\documentclass[conference]{IEEEtran}
\usepackage{times}

\usepackage[numbers]{natbib}
\usepackage{multicol}
\usepackage[bookmarks=true]{hyperref}
\usepackage{amssymb,amsfonts,bm}
\usepackage{amsmath,tikz}
\usepackage{physics}
\usepackage{mathtools}
\usepackage{subcaption}
\usepackage{multirow}
\usepackage[ruled, vlined, linesnumbered]{algorithm2e}

\newtheorem{thm}{Theorem}[section]
\newtheorem{defn}[thm]{Definition}
\newtheorem{lem}[thm]{Lemma}
\newtheorem{cor}[thm]{Corollary}

\newcommand*{\myset}[1]{\mathcal{#1}} % greek boldface
\newcommand*{\workspace}{\myset{Q}} % greek boldface
\newcommand*{\config}{q} % greek boldface
\newcommand*{\rthree}{\mathbb{R}^3}
\newcommand*{\rtwo}{\mathbb{R}^2}
\newcommand*{\rone}{\mathbb{R}}
\newcommand*{\rpositive}{\mathbb{R}^+}

\newcommand*{\height}{h}
\newcommand*{\aug}[1]{\hat{#1}} % vector boldface
\newcommand*{\robpath}{q} % vector boldface
\newcommand*{\ra}{\rightarrow} % vector boldface
\newcommand*{\timet}{T} % vector boldface
\newcommand*{\sptime}{\xi} % vector boldface
\newcommand*{\smallheight}{\height_s} % vector boldface
\newcommand*{\pathfloor}{\Tilde{\robpath}} % vector boldface
\newcommand*{\projang}{\alpha} % vector boldface
\newcommand*{\projplane}{\myset{P}} % vector boldface
\newcommand*{\btau}{\tau} % vector boldface
\newcommand*{\bgen}{\sigma} % vector boldface
\newcommand*{\word}{b} % vector boldface
\newcommand*{\totalnum}{K} % vector boldface
\newcommand*{\triang}{\Delta} % vector boldface
\newcommand*{\entangle}{\gamma} % vector boldface
\newcommand*{\entangledef}{\phi} % vector boldface
\newcommand*{\angleset}{\myset{A}} % vector boldface
\newcommand*{\numsample}{m} % vector boldface

\newcommand{\overbar}[1]{\mkern 1.5mu\overline{\mkern-1.5mu#1\mkern-1.5mu}\mkern 1.5mu}
 
\newcommand*{\positiveinteger}{\mathbb{Z}^+} 
\newcommand*{\identity}{e} 
\newcommand*{\dummyc}{c} 
\newcommand*{\dummyf}{f} 
\newcommand*{\dummyl}{l} 
\newcommand*{\dummyg}{g} 
\newcommand*{\permpos}{p} 
\newcommand*{\permposset}{\Phi} 
\newcommand*{\permactionset}{\myset{U}} 
\newcommand*{\permaction}{u} 
\newcommand*{\braidgroup}{B} 
\newcommand*{\bset}{\myset{B}} 
\newcommand*{\node}{\text{node}} 
\newcommand*{\childnode}{\text{childNode}} 
 
\newcommand*{\mapping}{\theta}
\newcommand*{\polyapprox}{\myset{C}}
\newcommand*{\stringx}{X}
\newcommand*{\tcross}{\timet^{\text{cro}}_i}
\newcommand*{\smallv}{\epsilon}

\begin{document}

\title{Path Planning for Multiple Tethered Robots \\Using Topological Braids}

% You will get a Paper-ID when submitting a pdf file to the conference system
%\author{Author Names Omitted for Anonymous Review. Paper-ID 161}

\author{\authorblockN{Muqing Cao$^1$, Kun Cao$^1$, Shenghai Yuan$^1$, Kangcheng Liu$^1$, Yan Loi Wong$^2$, and Lihua Xie$^{1\ast}$}
\authorblockA{$^1$School of Electrical and Electronic Engineering, 
Nanyang Technological University, Singapore.\\
$^2$Department of Mathematics, Faculty of Science, National University of Singapore, Singapore.\\
$^\ast$ Corresponding author, Email: elhxie@ntu.edu.sg}
%\and
%\authorblockA{Twentieth Century Fox\\
%\and
%\authorblockN{James Kirk\\ and Montgomery Scott}
%\authorblockA{Starfleet Academy\\
%San Francisco, California 96678-2391\\

% avoiding spaces at the end of the author lines is not a problem with
% conference papers because we don't use \thanks or \IEEEmembership

%Homer Simpson\authorrefmark{2},
%Montgomery Scott\authorrefmark{3} and
%Eldon Tyrell\authorrefmark{4}}
%Telephone: (800) 555--1212, Fax: (888) 555--1212}
%\authorblockA{\authorrefmark{4}Tyrell Inc., 123 Replicant Street, Los Angeles, California 90210--4321}
}

\maketitle

\begin{abstract}
Path planning for multiple tethered robots is a challenging problem due to the complex interactions among the cables and the possibility of severe entanglements.
Previous works on this problem either consider idealistic cable models or provide no guarantee for entanglement-free paths.
In this work, we present a new approach to address this problem using the theory of braids.
By establishing a topological equivalence between the physical cables and the space-time trajectories of the robots, and identifying particular braid patterns that emerge from the entangled trajectories, 
we obtain the key finding that all complex entanglements stem from a finite number of interaction patterns between $2$ or $3$ robots.
Hence, non-entanglement can be guaranteed by avoiding these interaction patterns in the trajectories of the robots.
Based on this finding, we present a graph search algorithm using the permutation grid to efficiently search for a feasible topology of paths and reject braid patterns that result in an entanglement.
We demonstrate that the proposed algorithm can achieve $100\%$ goal-reaching capability without entanglement for up to 10 drones with a slack cable model in a high-fidelity simulation platform.
%Realistic simulations using up to $10$ robots with a slack cable model show that the proposed approach is able to achieve $100\%$ success in targets reaching tasks without any entanglements.
The practicality of the proposed approach is verified using three small tethered UAVs in indoor flight experiments.
%The proposed approach is verified in a realistic simulation using $6$ to $10$ robots with a slack cable model.
%The result shows that 
\end{abstract}

\IEEEpeerreviewmaketitle
\section*{Supplementary Material}
A video illustrating the simulation and experiments is available at
https://youtu.be/igP7eaOyZuc. The supplementary document and the source code can be found
at https://github.com/caomuqing/tethered\_robots\_path\_planning.

\section{Introduction}
Tethered robots are connected to fixed or mobile objects via tether cables \cite{Tognon2017tether}.
Depending on the applications, a tether cable may supply uninterrupted power to a robot, ensure a robust communication link, or act as a physical connection to an item for transportation.
Despite the benefits, a cable is prone to entanglements with surrounding obstacles, which may greatly limit the reachable space of the robot and even cause collisions.
Therefore, the path planning of tethered robots is an important topic to ensure the safety of the operations.
%Due to their ability to operate for an extended duration and in regions with wireless interference, tethered systems have been widely used in industries.
Path planning of a single tethered robot has been well studied by the research community and efficient algorithms have been proposed to navigate a tethered robot around the obstacles in a planar or 3-D environment \cite{Teshnizi2014, Yang2022,kim2014path}.
Recently, collaborative tethered robots have also been studied for applications such as 
search and exploration \cite{Shapovalov2020Exploration, Petit2022},
object gathering and removal \cite{Su2022, Bhattacharya2015}, and
item transportation \cite{Kotaru2020}.
Despite increasing interest in the path planning of multiple tethered robots, it is still a challenging problem due to the complex interactions among the cables and the difficulty in modeling the entanglement.

\begin{figure}
	\centering
    \subcaptionbox{ \footnotesize Using the proposed approach, the robots remain untangled. \label{fig: noentangle}}[0.99\linewidth]{\includegraphics[width=0.85\linewidth]{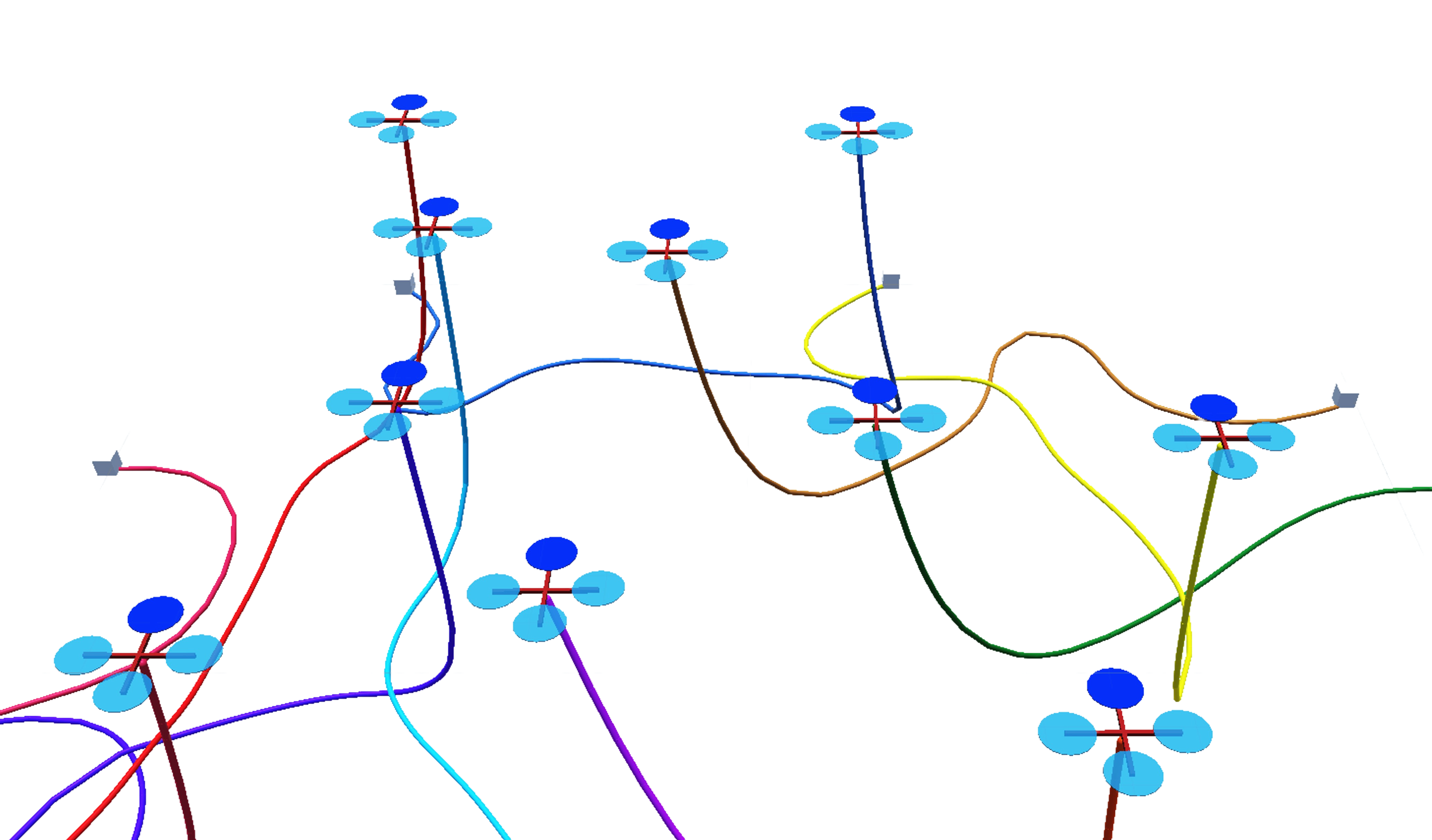}}\\
    \subcaptionbox{\footnotesize Using a baseline approach, the cables become severely entangled. \label{fig: gotentangle}}[0.99\linewidth]{\includegraphics[width=0.85\linewidth]{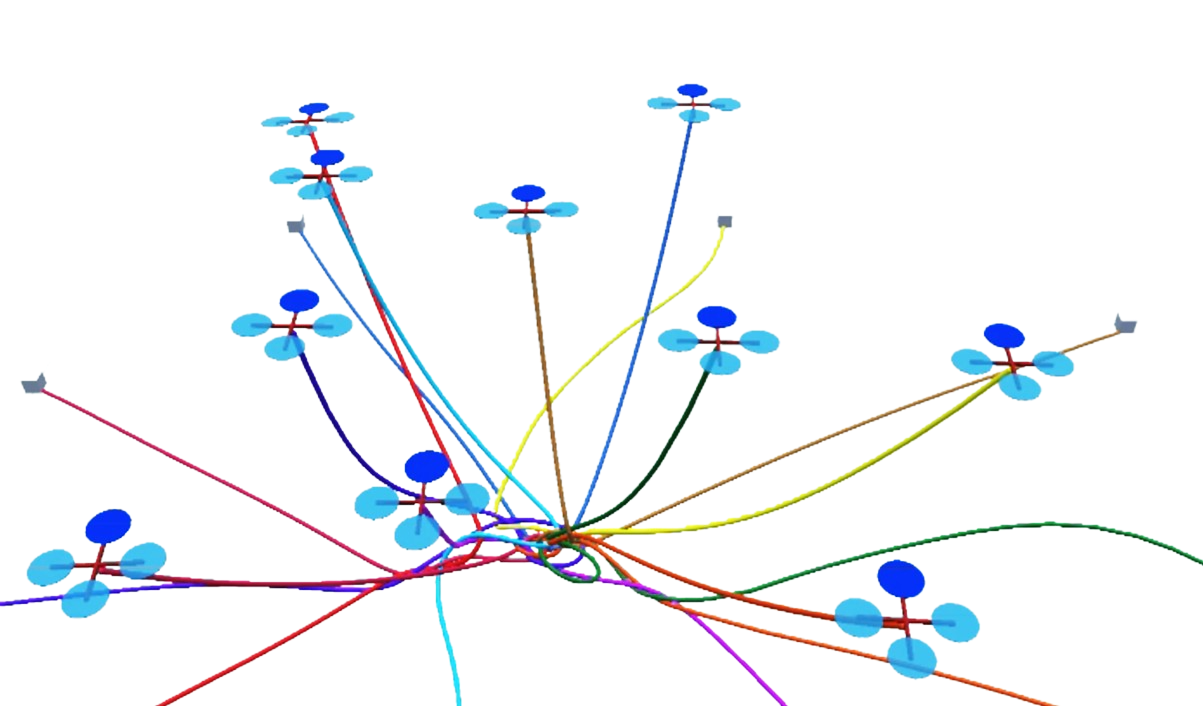}}
	\caption{\footnotesize Simulations of multiple tethered UAVs to reach random targets using (a) the proposed approach and (b) a baseline approach that does not take tethers into consideration.
    }  \label{fig: simulation}
\end{figure}

Existing works address this problem by restricting the problem settings or simplifying the cable model.
%The focus of this work is the path planning of multiple tethered robots connected to fixed base stations in a bounded workspace.
\citet{Sinden1990} considers a planar workspace and focuses on finding a permissible sequence of visiting the targets such that the straight cables do not cross each other.
A planar workspace is also considered in \citet{Zhang2019}, but the robots are allowed to push the other cables when in contact. 
\citet{Rajan2016} propose an entanglement detection system for a chain of tethered robots, which requires additional hardware on each robot for the measurement of tension and angles.
\citet{hert1999motion,hert1996ties} consider the navigation of multiple robots in a 3-D workspace with fully stretched cables, 
and define entanglement as any bending due to cable-cable contacts.
A movement of a robot results in a triangular area swept by the straight cable, hence feasible paths are found by checking intersections between the swept area and the other cables.
In practice, cables are hardly fully straight, hence such an approach does not guarantee collision avoidance and non-entanglement.
The recent work \cite{cao2022neptune} presents a distributed approach for trajectory planning of multiple tethered robots with consideration for slack cables.
Relying on a topology-guided heuristic that records the crossings among the cables, the approach generates feasible paths in an efficient manner.
However, the approach does not guarantee non-entanglement and falls into deadlocks when the number of robots increases.

The theories of knots and braids are important topics in the field of low-dimensional topology \cite{prasolov1997knots} 
%The manipulation of topological knots and braids to obtain their simpler forms
%can be naturally associated with the act of untangling cables in real life, 
and have seen recent applications in robotic systems to fold and unfold physical knots.
Disentangling one or multiple cables using robot arms is studied in \cite{Yan2020,Shivakumar-RSS-22}.
\citet{Antonio2022folding} plan paths for a team of unmanned aerial vehicles (UAVs) to form a desired knot pattern using a long cable.
The results of these works are not applicable to our problem, as they allow grasping and pulling at multiple locations along a cable, while a tethered robot is only connected to the end of a cable.
Braid theory has also been applied in recent works to characterize the topology of the interactions among moving robots\cite{Diaz2017multirobot, Mavrogiannis2019}.
However, the connection between braids and tethered robots remains unrevealed.

In this work, we aim to answer the following questions: (1) is the entanglement of the cables in a multi-robot scenario associated with special topological patterns in the braids? (2) 
can non-entangling paths be generated for multiple tethered robots in a bounded workspace, considering a slack cable model?
We first provide a formal definition of entanglement based on the concepts of isotopy and elementary moves.
By introducing a parameter that defines the allowable bending in the cables, our definition of entanglement is applicable to both slack and taut cables.
To answer the first question, we establish the topological equivalence between the cables and the space-time trajectories of the robots. 
Then, by acquiring a topological characterization of the entangled space-time trajectories using braids, 
we identify particular braid patterns necessary for the occurrence of entanglements.
The key insight is that any entanglements, however complex, are resulted from a few interaction patterns between $2$ or $3$ robots. 
To address the second question, we propose a graph search algorithm that searches for a feasible topology of paths using the concept of permutation grids.
The algorithm efficiently rejects path topologies that result in entangling braid patterns and hence guarantees non-entanglement for the generated paths.
%to generate guaranteed non-entangling paths avoiding the entangling braid patterns.
The proposed algorithm is evaluated in a simulation involving $6$ to $10$ robots.
Comparisons with the existing approaches show that our approach is the only one that completes all tasks successfully.
The main contributions of this work are summarized as follows:
\begin{itemize}
    \item We present a formal definition of entanglement for multiple tethered robots applicable to both taut and slack cable models;
    \item We identify the braid patterns necessary for the occurrence of entanglement and establish the conditions for generating non-entangling trajectories;
    \item A permutation grid search algorithm is proposed to generate guaranteed non-entangling paths considering a slack cable model;
    \item The effectiveness of the algorithm in entanglement prevention is verified in realistic simulations and comparison with the existing approaches.
    \item Flight experiments using three UAVs verify the practicality of the approach in real tethered systems.
\end{itemize}
To the best of our knowledge, this is the first work that addresses the path planning of multiple tethered robots with guaranteed non-entanglement using a slack cable model. 

The rest of this paper is organized as follows. Notations and preliminary concepts related to isotopy and braids are discussed in Section \ref{sec:prelim}. 
In Section \ref{sec:braids}, we introduce a procedure to obtain a topological characterization of entanglements using the theory of braids and present detailed proofs.
Section \ref{sec:planning} presents the path planning algorithm using permutation grids.
Section \ref{sec:simulation} introduces the simulation setup and discusses the simulation results.
Flight experiments using small UAVs are presented in Section \ref{sec: exp}.
%Section \ref{sec: discussion} discusses the limitations and potential applications of the proposed approach. 
Section \ref{sec:conclusion} draws the conclusion.

\section{Preliminaries}\label{sec:prelim}
\subsection{Notation}
In this paper, $\mathbb{R}^n$ denotes the $n$-dimensional Euclidean space,
$\positiveinteger$ indicates the set of positive integers.
$\myset{I}_{n}$ denotes the set consisting of integers $1$ to $n$, i.e., $\myset{I}_{n} = \{1,\dots,n\}$.
A line segment with two boundary points $a$ and $b$ is denoted by $\overbar{ab}$.
More symbols will be introduced when they appear in the paper.

\subsection{Elementary Moves and Isotopy}\label{subsec: elementary}
\begin{figure}[!t]
\centering
\includegraphics[width=0.8\linewidth]{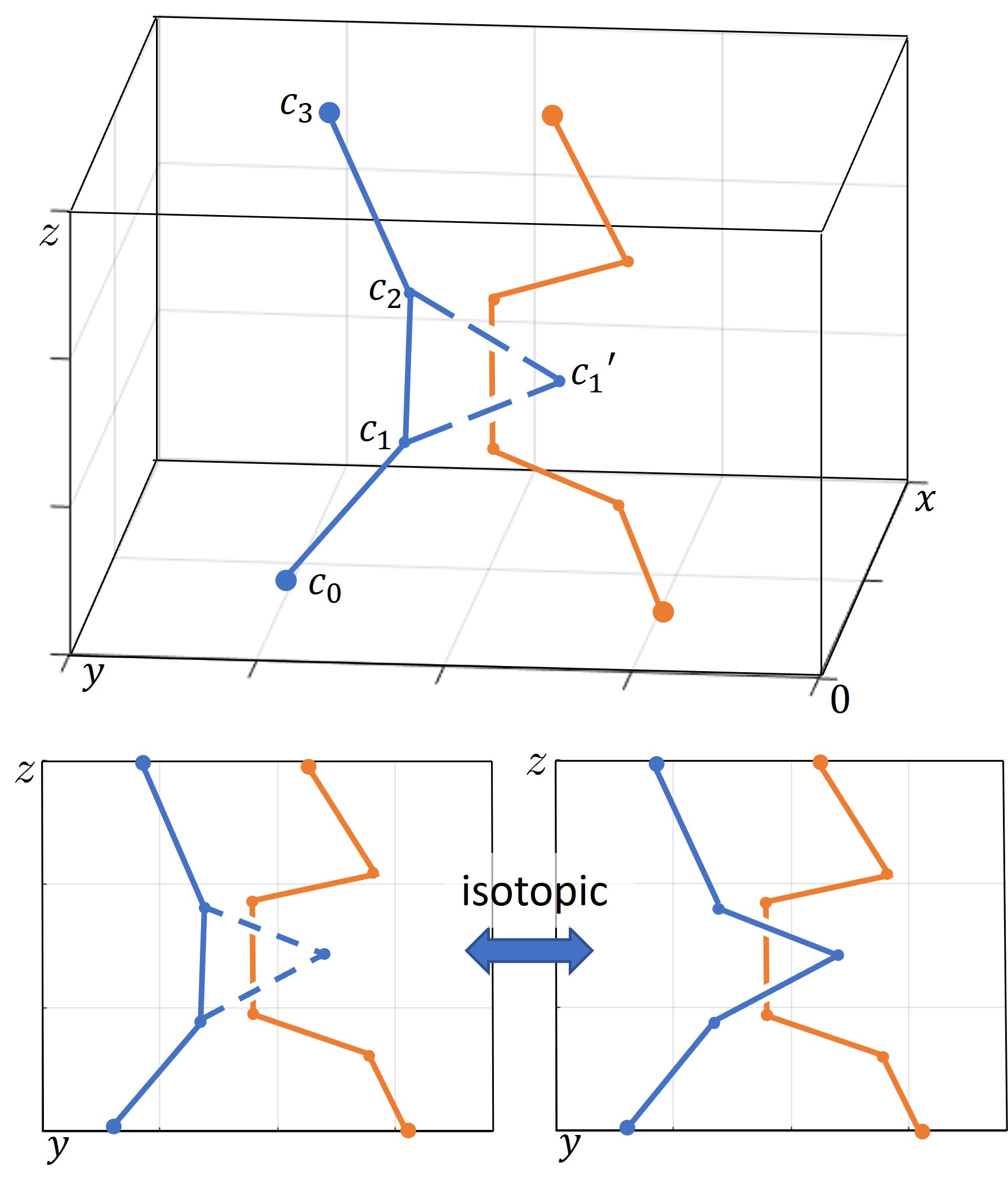}
\caption{\footnotesize An illustration of an elementary move.}
\label{fig: elementary}
\end{figure}
In this work, we consider a 3-Dimensional Euclidean space bounded by two horizontal planes, $\aug{\workspace}=\{(x,y,z)|(x,y)\in\workspace,0\leq z\leq \height\}$, where $\workspace\subset\rtwo$ is a simply connected 2-D region, $\height$ is the height of the workspace.
Denote the intersection between $\aug{\workspace}$ and the level plane $z=\dummyl\in[0,h]$ as $\aug{\workspace}_{\dummyl}$, i.e.,
$\aug{\workspace}_{\dummyl}=\{(x,y,z)|(x,y)\in\workspace,z=\dummyl\}$.
Consider a set of non-intersecting continuous curves, each starting from the floor of the workspace, $\aug{\workspace}_0$, and ending at the ceiling of the workspace, $\aug{\workspace}_\height$.
A polygonal approximation of the curves is a set of polygonal lines that shares the same starting and ending points with the original curves, and can be continuously deformed into the original curves without intersecting each other.
Consider $\overbar{\dummyc_1\dummyc_2}$ to be an edge on a polygonal chain, as shown in Figure \ref{fig: elementary}. Let $\dummyc_1'$ be a point in $\aug{\workspace}$ such that the triangle $\triang\dummyc_1\dummyc_1'\dummyc_2$ does not intersect with any other polygonal chains. 
%Furthermore, the edges $\overbar{\dummyc_1\dummyc_1'}\cup\overbar{\dummyc_1'\dummyc_2}$ intersect with any horizontal plane $\aug{\workspace}_\dummyl$ at one point at most, $\forall\dummyl\in[0,h]$.
An elementary move is an operation that replaces $\overbar{\dummyc_1\dummyc_2}$ by $\overbar{\dummyc_1\dummyc_1'}\cup\overbar{\dummyc_1'\dummyc_2}$, or in the case that $\overbar{\dummyc_1\dummyc_1'}\cup\overbar{\dummyc_1'\dummyc_2}$ is part of the original chain, replace it by $\overbar{\dummyc_1\dummyc_2}$ \cite{prasolov1997knots}.
\begin{defn}[Isotopy]
    Two sets of polygonal lines in $\aug{\workspace}$ are isotopic or ambient isotopic if one set of lines can be transformed into the other through a sequence of elementary moves.
\end{defn}

%Elementary moves and isotopy can be also defined in a 2-D diagram.
Consider a projection of polygonal lines onto a plane perpendicular to the X-Y plane. 
At the intersections between polygonal lines, overpasses and underpasses are defined based on their spatial relations in 3-D.
An elementary move in 3-D has a corresponding elementary move in 2-D, as shown in the bottom left of Figure \ref{fig: elementary}.
Similarly, two sets of projected polygonal lines in 2-D are isotopic or plane isotopic, if a sequence of 2-D elementary moves can be applied to transform one to another.

%Each curve can be approximated as a polygonal chain with vertices, $\dummyc_0,\dummyc_1,\dots$, such that the approximated polygonal chain can be continuously deformed into the original curve without intersecting the approximated

\subsection{Topological Braids}
The Artin $n$-braid group, denoted as $\braidgroup_n$, is a group with $n-1$ generators $\bgen_1, \bgen_2,\dots,\bgen_{n-1}$ and the group relations \cite{Kassel2008}
\begin{align}
    \bgen_i\bgen_j=\bgen_j\bgen_i, \; i,j\in\myset{I}_{n-1},|i-j|\geq2,\\
    \bgen_i\bgen_{i+1}\bgen_i=\bgen_{i+1}\bgen_{i}\bgen_{i+1},\;i\in\myset{I}_{n-2}.\label{eq: braidgroup}
\end{align}
The identity element in the group is denoted as $\identity$.
$\braidgroup_2$ is generated by a single generator $\bgen_1$ with no group relations, and $\braidgroup_3$ is generated by $\bgen_1, \bgen_2$ and relation (\ref{eq: braidgroup}).
A braid, $\word\in\braidgroup_n$ can be written as a composition of group generators and their inverses, $\word=\btau_1\btau_2\dots\btau_\totalnum$, where $\totalnum$ is the length of the braid, $\btau_i\in\{\bgen_1^{\pm1},\bgen_2^{\pm1},\dots,\bgen_{n-1}^{\pm1}\}$ is called an elementary braid.
%Two braids in a $n$-Braid group are equivalent if their

A curve or a polygonal line is called ascending if it is monotonically increasing in $z$, in other words, each horizontal plane intersects with an ascending line at only one point.
An $n$-braid can be represented in a 2-D diagram consisting of $n$ ascending strings $\stringx_i(z):[0,1]\rightarrow\rone$, $i\in\myset{I}_n$.
The starting and ending points of each string satisfy $\stringx_i(0)\in\myset{I}_n$, $\stringx_i(1)\in\myset{I}_n$.
Each elementary braid $\bgen_i^{\pm1}$ in the braid word corresponds to a crossing between the $i$-th string ($i$ denotes the order of the string when counting from left to right) and the $(i+1)$-th string, where an overpass by the $i$-th string is denoted as $\bgen_i$ and the underpass is denoted as $\bgen_i^{-1}$.
%Two braid diagrams are isotopic if they

In the standard definition of braids, each braid string is only defined in the domain $\{z\in[0,1]\}$.
In this work, we relax the definition by allowing the braid strings to have a domain $[0,t]$ for $t\in\rpositive$.
Furthermore, $\word(t)$ indicates the braid obtained when the crossings among the braid strings in the interval $[0,t]$ are taken into account.
Examples of braid diagrams are shown in the bottom left of Figure \ref{fig: overview}.

\section{Topological Characterization of Entanglements Using Braids}\label{sec:braids}
%We consider a 3-Dimensional Euclidean space $\aug{\workspace}=\{(x,y,z)|(x,y)\in\workspace,0\leq z\leq \height\}$, where $\workspace\in\rtwo$ is a simply connected 2-D region, $\height$ is the height of the workspace.
We consider a team of $n$ tethered robots navigating in the workspace $\aug{\workspace}$, 
and assume the robots' movements to be constrained in the ceiling of the workspace $\aug{\workspace}_\height$.
%=\{(x,y,\height|(x,y)\in\workspace\}$.
To reach a target position at a different height, a robot first moves to the same horizontal position, then descends to the target.
Each robot is attached to a base station placed on the floor $\aug{\workspace}_0$.
%=\{(x,y,0)|(x,y)\in\workspace\}$ through a cable. 
%A robot's initial position is directly above its base.
% \begin{ass}{(Robot operation heights)}\label{ass: robot heights}
% Robots maintain the same height $z$ during their movements
% \end{ass}
%The cables are continuous curves with a small thickness.
The cables form a set of mutually disjoint topological intervals that start at the bottom of the workspace and end at the ceiling, as shown in the top left of Figure \ref{fig: overview}.
A robot follows a path $\robpath_i:[0,\timet_i]\ra\workspace$, where $\robpath_i(0)=\config_i^s$ is the same as the horizontal position of its base and $\robpath_i(\timet_i)=\config_i^d$ is a user-defined target.
A scaled space-time trajectory of the robot is constructed as $\sptime_i:[0,\timet]\ra\aug{\workspace}$, where $\sptime_i(t)=(\robpath_i(t),t\frac{\height}{\timet})\in\rthree$ for $0\leq t<\timet_i$ and $\sptime_i(t)=(\robpath_i(\timet_i),t\frac{\height}{\timet})$ for $\timet_i\leq t\leq \timet$.
$\timet=\max_{i\in\myset{I}_{n}}\timet_i$ is the longest time taken by any robot to reach the target.
At a height $z$, $0\leq z\leq\height$, the collection of scaled space-time trajectories $\{\sptime_i\}_{i\in\myset{I}_n}$ intersects with $\aug{\workspace}$ at $n$ distinct points (given that the robot's trajectories are not in a collision). See the top right of Figure \ref{fig: overview} for an illustration of scaled space-time trajectories.

\begin{figure*}[!t]
\centering
\includegraphics[width=1.0\linewidth]{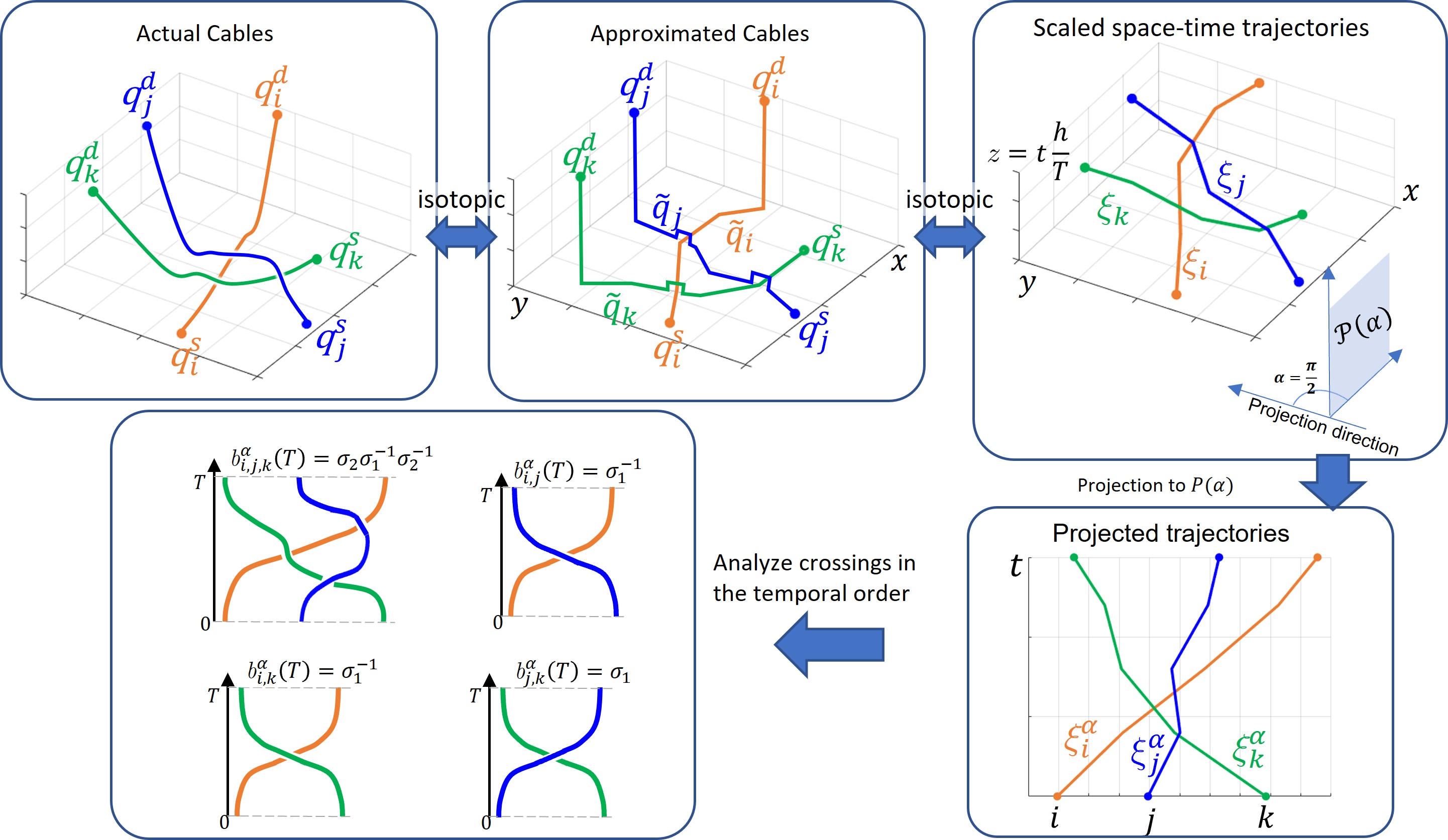}
\caption{\footnotesize Overview of the approach.}
\label{fig: overview}
\end{figure*}

\begin{lem}\label{lm: isotopy}
The set of cables connecting $n$ robots to their bases is isotopic to the scaled space-time trajectories of the robots.
\end{lem}
\begin{proof}
The shapes of the cables are closely related to the paths taken by the robots in the worksapce, because (1) a cable hanging from a robot will likely have its first contact with the ground in the neighbourhood of the X-Y coordinates of the robot, (2) when robot $i$ crosses a path that has taken by robot $j$, robot $i$'s cable will slide over the cable of robot $j$.
Therefore, we construct an approximation of the configurations of the cables, labeled as $\pathfloor_i\in\rthree$, in the following way (graphic illustration in the top middle of Figure \ref{fig: overview})
\begin{align}
    \pathfloor_i(t) = 
    \begin{cases}
        (\robpath_i(t),0),& t\in\timet_i\backslash\tcross\\
        (\robpath_i(t),\smallheight), & t\in\tcross\\
        (\robpath_i(\timet_i),\height), &t>\timet_i.
    \end{cases}    
\end{align}
$\tcross$ denotes a set of time intervals, each time interval is a small neighbourhood of the time that robot $i$ travels to a same location visited by another robot before, i.e.,
$\tcross=\{[t-\smallv, t+\smallv]|\robpath_i(t)=\robpath_j(t_j),\forall t\in(0,\timet_i], t_j\in(0,t), j\in\myset{I}_n\backslash i\}$.
$\smallheight$ is a value greater than zero, indicating a small height that a cable is elevated to. 
Clearly, the set $\{\pathfloor_i\}_{i\in\myset{I}_n}$ is isotopic to the actual cables of the robots. 
%We construct an approximated cable that joins the base and the robot by first following the horizontal path of the robot in a region near the floor, $\pathfloor_i(t)=(\robpath_i(t),z)\in\rthree$ for $0\leq z\leq\smallheight$, 
%$\smallheight$ is a small value.
%The approximated cable then follows a vertical line joining $\pathfloor_i(\timet_i)$ and $\sptime_i(\timet_i)$, see the top middle of Figure \ref{fig: overview}.
%At the intersection between robot paths, the path that arrives at the location at a later time is elevated to a slightly higher altitude, which resembles a cable sliding over the surface of another cable.
%The set of approximated cables is isotopic to the actual cables, which are pulled toward the ground by gravity and represent the same topology as the paths taken by the robots.
%Clearly, a collection of the approximated paths is isotopic to the actual cables of the robots.
To establish an isotopy between $\{\pathfloor_i\}_{i\in\myset{I}_n}$ and the space-time trajectories, note that we can transform $\pathfloor_i(t)$ to $\sptime_i(t)$ by elementary moves for all $t\in[0,\timet]$, 
because they share the same X-Y coordinates for all $t$, and their order in the $z$-coordinates ($z$-order) are the same, i.e., for $\pathfloor_i(t)$ and $\pathfloor_j(t_j)$ such that $\robpath_i(t)=\robpath_j(t_j)$, $t>t_j$, $\pathfloor_i(t)$ has a higher $z$-coordinate than $\pathfloor_j(t_j)$, 
$\sptime_i(t)$ also has a higher $z$-coordinate than $\sptime_j(t_j)$. 
This is because a robot who travels to the same location at a later time has its space-time trajectory at a higher $z$-coordinate.
%because the region bounded by 
%, except for small neighborhoods of the intersection points of the paths.
%In order to elevate the cables in an overlapping region, we need to ensure that the $z$-order of the approximated cables at the intersection is the same as the $z$-order of their corresponding space-time trajectories.
%Indeed, a robot who travels to the same location at a later time has its space-time trajectory at a higher $z$-coordinate.
Hence, the cables can be transformed to their corresponding space-time trajectories isotopically.
\end{proof}

We specify a 2-D plane perpendicular to the X-Y plane as $\projplane(\projang)=\{(x,y,z)|x\cos\projang+y\sin\projang=0,z\in\rone\}$ where $\projang\in[0,\pi]$ is the projection angle with respect to the positive X axis. 
A set of 2-D trajectories 
$\sptime^{\projang}_{i}:[0,\timet]\ra\projplane(\projang)$, $i\in\myset{I}_n$, 
is obtained by the projection of the space-time trajectories $\sptime_i$ onto $\projplane(\projang)$ (bottom right of Figure \ref{fig: overview}).
%We assume that the projected trajectories satisfy the following conditions to be a regular braid diagram \cite{}: 
%(1) no more than two projected trajectories intersect at the same point; (2) a vertex of a polygonal trajectory does not coincide with any other projected trajectories.
%We assume that no more than two trajectories intersect at the same point in a projection plane $\projplane(\projang)$,
%Such a case occurs when more than two robots lie on the same line perpendicular to $\projplane(\projang)$, 
%which is rare and can be easily avoided by varying $\projang$.

A crossing between 2-D trajectories indicates an event of two robots swapping positions in the projected axis.
Such an event can be represented as a braid generator $\bgen_i^{\pm1}$, where $i$ indicates the ranking of the leftmost swapping robot in increasing order of the robots' projected positions. 
%The sign of the braid generator depends on the relative positions of the swapping robots in the axis perpendicular to $\projplane$.
A braid word $\word^{\projang}(t)$ is obtained by joining the elementary braids representing the crossing events that have occurred from time $0$ to time $t$.
%: $\word^{\projang}(t)=\btau_1\btau_2\dots\btau_j\dots\btau_\totalnum$, where $\btau_j\in\bgen_i^{\pm1}$ indicates the $j$-th crossing event, $\totalnum$ is the total number of crossings up to time $t$.
%A 3-braid and a $2$-braid denoted as $\word^\projang_{i,j,k}(t)$ and $\word^\projang_{i,j}(t)$, 
%are the braid words obtained when only the crossings between the specified robots are considered.
%are obtained from $\word^{\projang}(t)$ by removing all the trajectories except for the trajectories of robots $i,j,k$, or robots $i,j$.
Let $\word^\projang_{i,j,k}(t)$ be the $3$-braid obtained by removing all the trajectories except for the trajectories of robots $i,j,k$.
Similarly, $\word^\projang_{i,j}(t)$ indicates the $2$-braid obtained when only considering the crossings between robots $i$ and $j$.
Here, $i,j,k\in\myset{I}_n$ are the fixed indices of the robots.
%To obtain a braid word, a valid projection plane has to be chosen such that no more than two trajectories intersect at the same point. 
%Such case occurs when more than two robots lie on the same line perpendicular to $\projplane(\projang)$, which is rare and can be easily avoided by varying $\projang$.
For each braid word, an equivalent braid diagram can be drawn, as shown in the bottom left of Figure \ref{fig: overview}. 

We have introduced a procedure to obtain a topological characterization of robot paths in the form of braids. To establish a connection between the entanglements of cables and the topological braids,
we first provide a formal definition of the entanglement based on the horizontal bending angles of the cables.
%we first show the topological equivalence between the cables and the space-time trajectories of robots in the following lemma.
As described in Section \ref{subsec: elementary}, the cables can be approximated as a set of non-intersecting polygonal segments, 
and elementary moves can be applied to shorten the length of the cables while preserving isotopy.
Another interpretation of this shortening process is that the base station exerts tension on the cable and retracts the cable while the robots hold their positions.
The cables are shortened until they are either completely straight or in contact with other cables.

\begin{defn}[Maximum angle of rotation]
    Given a poly-gonal approximation of the cables, denoted as $\polyapprox$, a set of projected line segments onto the X-Y plane can be obtained.
    Each segment is assigned a direction consistent with the direction from the base to the robot (Figure \ref{fig: ent_angle}).
    %The maximum angle of rotation of robot $i$, denoted as 
    $\entangle_i(\polyapprox)$ is the maximum angle of rotation between any of the projected segments of robot $i$, $\gamma_i\in[0,\pi]$.
    The maximum angle of rotation of the entire team for this particular polygonal approximation is $\entangle(\polyapprox)=\max_{i\in\myset{I}_n}\entangle_i(\polyapprox)$.
    The minimum of $\entangle$ among all isotopic polygonal approximations of the cables is denoted as $\underline{\entangle}$, i.e., $\underline{\entangle}=\min_{\polyapprox}\entangle(\polyapprox)=\min_{\polyapprox}\max_{i\in\myset{I}_n}\entangle_i(\polyapprox)$.
\end{defn}

Intuitively, 
$\entangle(\polyapprox)$ indicates the extent of deviation from a set of straight lines for a particular polygonal approximation $\polyapprox$, 
and $\underline{\entangle}$ indicates the minimum deviation possible, which usually occurs when the cables are fully retracted.
Only horizontal bending is considered because the degree of vertical bending is small when the robots move at a similar height.
Now, we give the definition of entanglement based on the bending angle.
\begin{figure}[!t]
\centering
\includegraphics[width=0.6\linewidth]{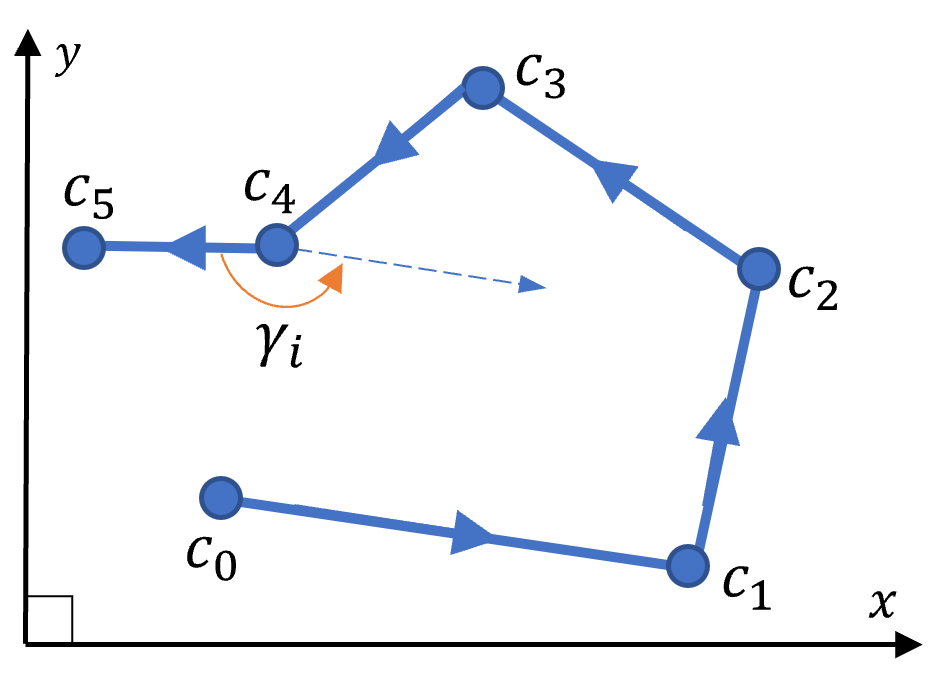}
\caption{\footnotesize The projected polygonal segments onto X-Y plane for robot $i$. The maximum angle of rotation, $\entangle_i$, is the rotation angle between $\overbar{c_4c_5}$ and $\overbar{c_0c_1}$. The dashed line is parallel to $\overbar{c_0c_1}$.}
\label{fig: ent_angle}
\end{figure}
\begin{defn}[$\entangledef$-Entanglement]
The cables are said to be $\entangledef$-entangled or in a state of $\entangledef$-entanglement when $\underline{\entangle}\geq\entangledef$ for a chosen $\entangledef\in(0,\pi]$.
%they do not have an isotopic polygonal approximation that satisfy $\lvert\entangle\rvert<\underline{\entangle}$, $\entangle\in[-\pi,\pi],\underline{\entangle}\in(0,\pi]$.
\end{defn}

\begin{figure}[!t]
\centering
\includegraphics[width=\linewidth]{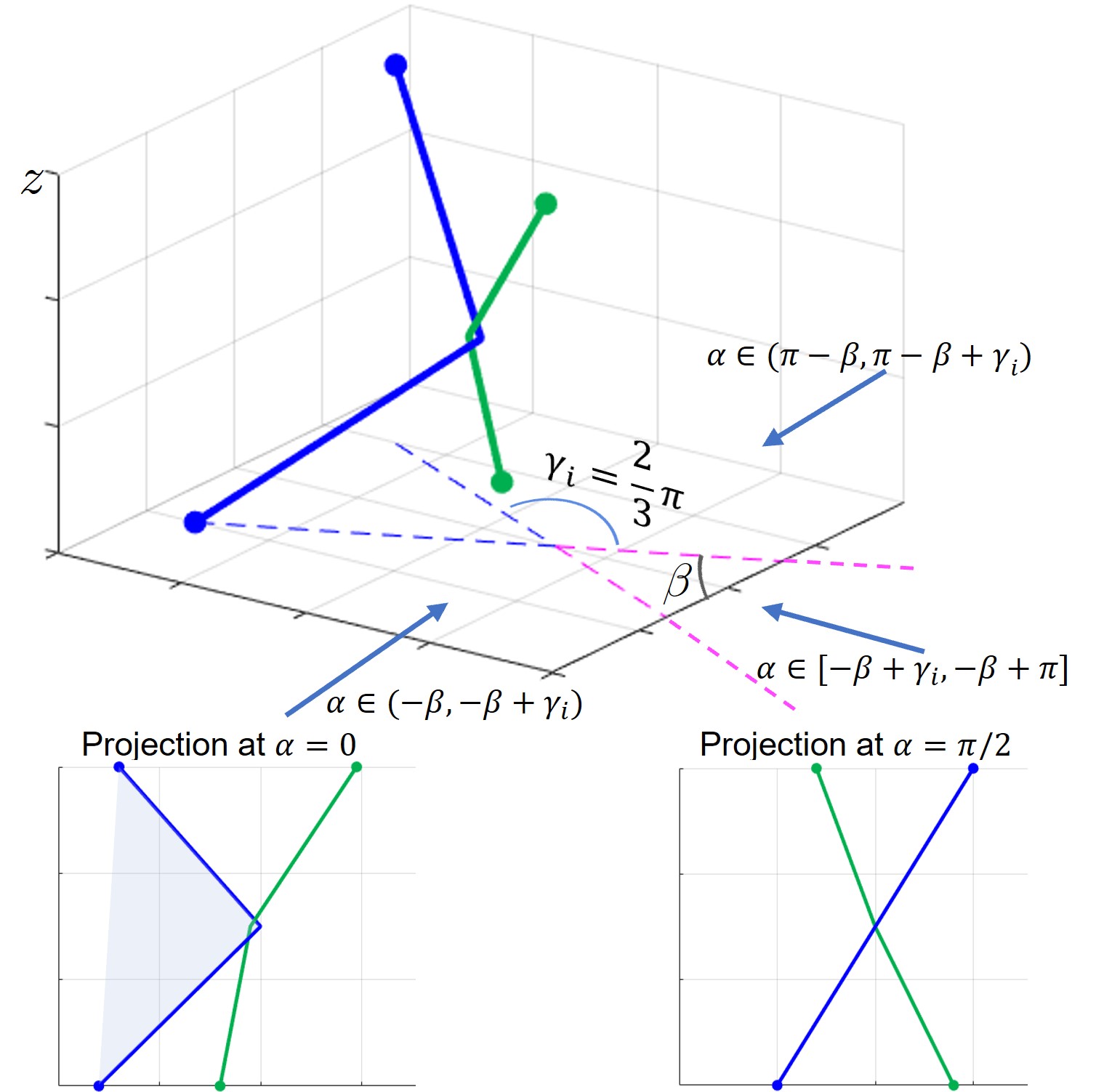}
\caption{\footnotesize An illustration of entanglement. The blue and green solid lines are the cables/trajectories of robot $i$ and $j$. The blue dashed lines are the cable/trajectory of robot $i$ projected onto the X-Y plane. The bottom two plots show the projections of the cables/trajectories onto a plane $\projplane(\projang)$. In the bottom left plot, the blue projected trajectory is non-monotonic.}
\label{fig: entangle}
\end{figure}

Figure \ref{fig: entangle} illustrates a polygonal approximation of two cables with $\entangle_i=\underline{\entangle}=\frac{2}{3}\pi$, thus the cables are $\frac{2}{3}\pi$-entangled.
%$\frac{2}{3}\pi$-entanglement between two robots.
%A higher value of $\underline{\entangle}$ indicates a higher degree of bending of the cables and hence a more severe entanglement.
%$\entangledef$ indicates the degree of tolerance for bending:
When $\entangledef$ is chosen close to zero, any small bending in the cables is considered an entanglement, which is in line with the definition of entanglement for taut cables in \cite{hert1999motion}.
Slack cables generally have a higher tolerance for bending, hence a higher $\entangledef$ may be chosen.
%For slack cables with a higher tolerance for bending, we may choose...
We neglect trivial cases of entanglement where cables are bent due to coplanarity by
assuming that $\{(\config_i^s,0),(\config_i^d,h)\}$ are not co-planar with $\{(\config_j^s,0),(\config_j^d,h)\}$, $\forall i,j\in\myset{I}_n$, $i\neq j$.

\begin{cor}\label{cor: spacetime}
If the cables are $\entangledef$-entangled, then the space-time trajectories of robots are also $\entangledef$-entangled, i.e., the space-time trajectories cannot be isotopically transformed to a polygonal approximation $\polyapprox$ such that $\entangle(\polyapprox)<\entangledef$.
\end{cor}

Owing to Corollary \ref{cor: spacetime}, the identification of entanglement can be done by analyzing the space-time trajectories.
In the following lemma, we show that given suitable projection angles, 
the projections of $\entangledef$-entangled trajectories exhibit a special property.
% \begin{lem}
%     A valid projection plane $\projplane(\projang)$ always can be found for a space-time representation of $n$ robots, where $n$ is finite.
% \end{lem}
% \begin{proof}
    
% \end{proof}
% \begin{defn}[Braid Entanglement]
%     A braid is entangled if every equivalent braid diagram consists of at least one braid string that intersects with another string more than once.
% \end{defn}
\begin{lem}\label{lm:straightline}
    Define $\angleset(\numsample)=\{\frac{i}{\numsample}\pi|i=0\dots m,m\in\positiveinteger\}$ to be the set of projection angles evenly dividing the range $[0, \pi]$.
    By setting $\numsample>\frac{\pi}{\entangledef}$, there exists $\alpha\in\angleset(\numsample)$, such that the projection of a set of $\entangledef$-entangled space-time trajectories onto the plane $\projplane(\projang)$ is non-isotopic to a set of straight lines.
    %such that the corresponding braid diagram of a set of $\underline{\entangle}$-entangled space-time trajectories consists of at least one of the sub-components shown in Figure.
\end{lem}
\begin{proof}
    See Section 1 of the supplementary document.
\end{proof}
In the following lemma, we show that a set of projected trajectories non-isotopic to straight lines can be identified by analyzing their corresponding $2$-braids and $3$-braids.
\begin{lem}\label{lm:braidword}
    For a set of projected trajectories $\{\sptime_i^{\projang}(t)\}_{i\in\myset{I}_n}$, $t\in[0,\timet]$, which is non-isotopic to a set of straight lines, 
    there exists a corresponding $3$-braid $\word^\projang_{i,j,k}(t)$ or $2$-braid $\word^\projang_{i,j}(t)$,
    $t\in(0,\timet]$, $i,j,k\in\myset{I}_n$, $i<j<k$, that satisfies at least one of the following:\\
    (1) $\word^\projang_{i,j}(t)$ is equivalent to $\bgen_1\bgen_1$ or $\bgen_1^{-1}\bgen_1^{-1}$;\\
    (2) $\word^\projang_{i,j,k}(t)$ is equivalent to a word in the set 
    $\{\bgen_\dummyf^{\dummyc}\bgen^{-\dummyc}_\dummyg\bgen_\dummyf^{\dummyc}|\dummyc\in\{1,-1\},\dummyf,\dummyg\in\{1,2\}, \dummyf\neq g \}$.
    %$\{\btau\cdot\bgen_\dummyf^{\dummyc}\bgen^{-\dummyc}_\dummyg\bgen_\dummyf^{\dummyc}|\dummyc\in\{1,-1\},\dummyf,\dummyg\in\{1,2\}, f\neq g, \btau\in\{\identity,\bgen_1^{\pm 1}\}\}$.

\end{lem}
\begin{figure}[!t]
\centering
\includegraphics[width=0.8\linewidth]{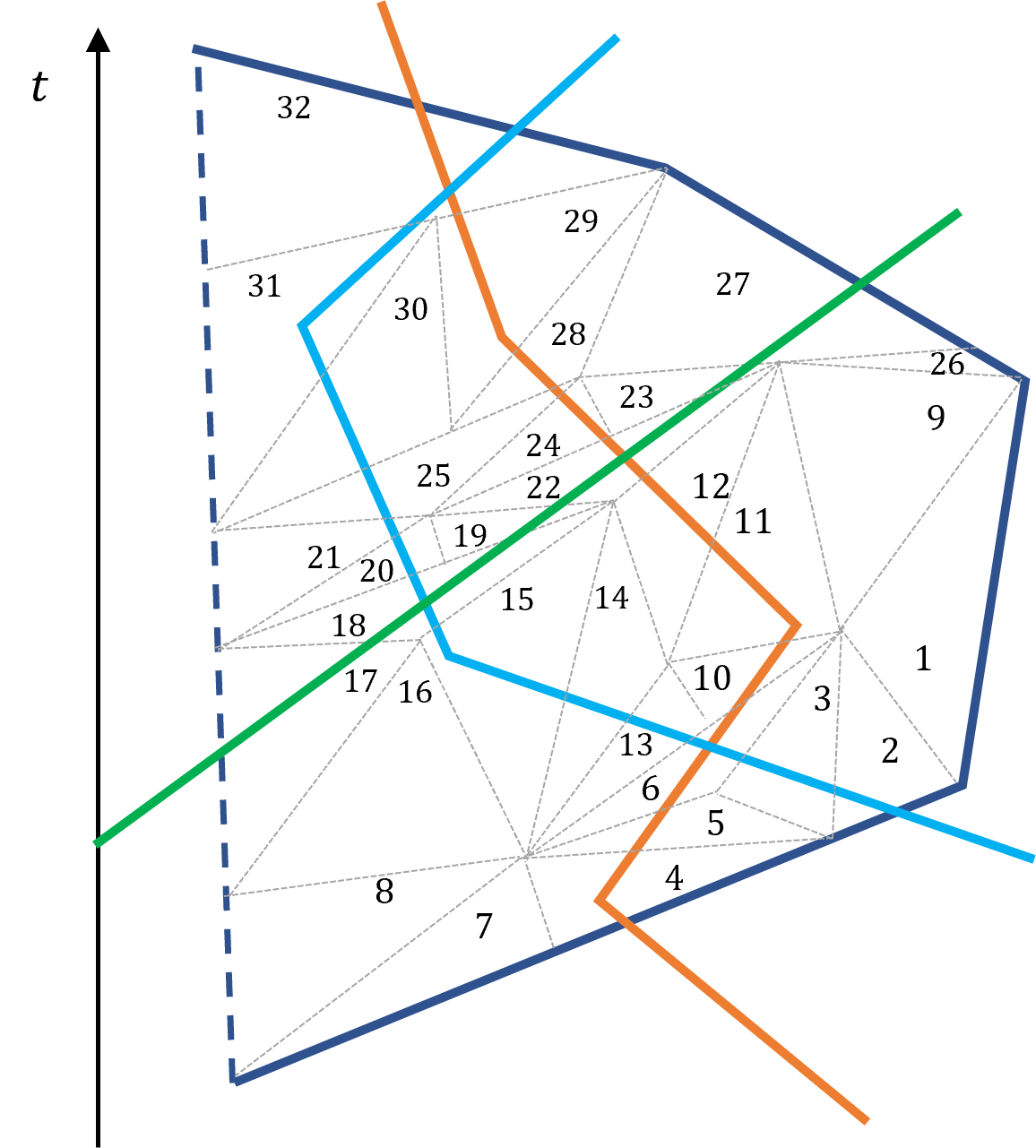}
\caption{\footnotesize The projected trajectories of robots. The gray dashed lines outline the partitioned triangles.}
\label{fig: straighten}
\end{figure}
\begin{figure}[!t]
\centering
\includegraphics[width=0.6\linewidth]{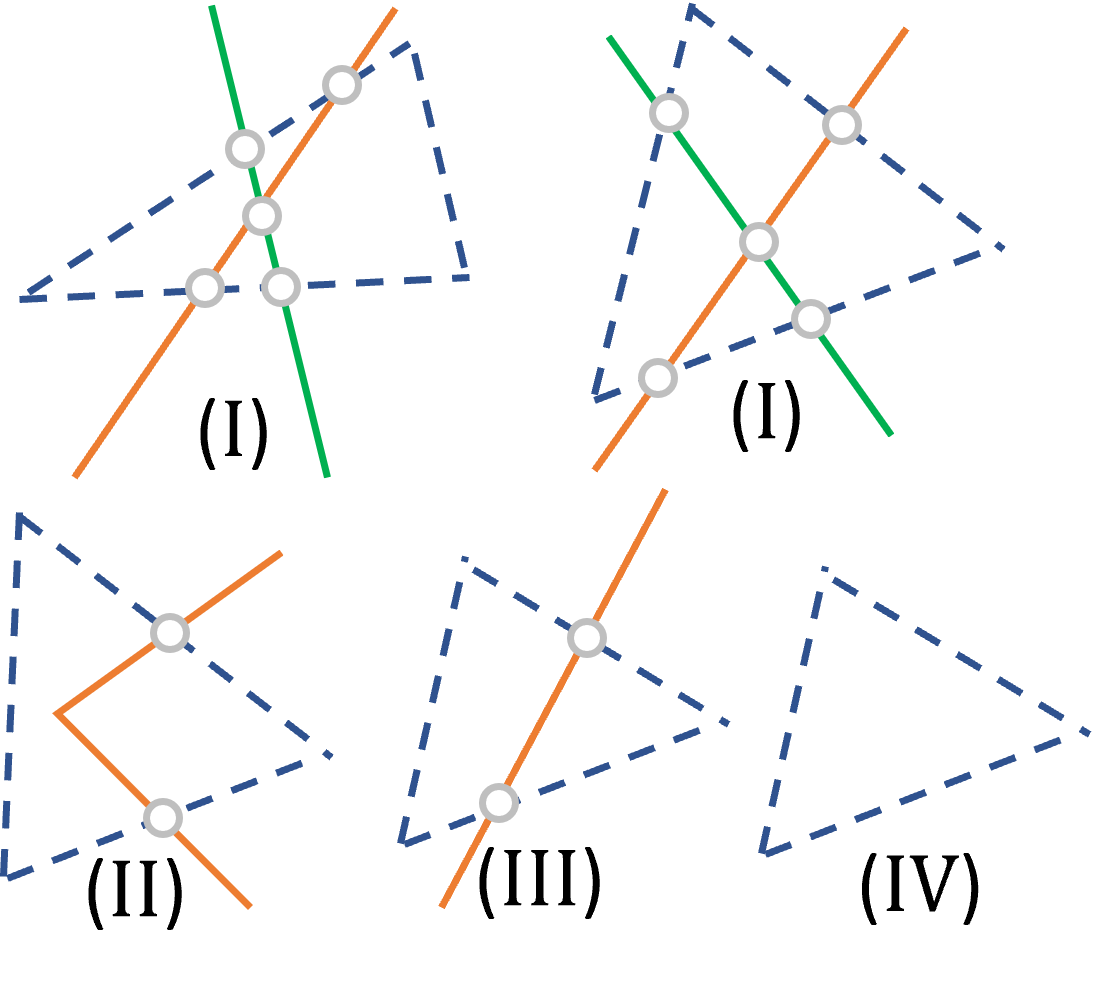}
\caption{\footnotesize Four types of triangles. The small circles indicate either an overpass or an underpass.
An elementary move may be executed from any one (respectively, two) edge of a triangle to the other two (respectively, one) edges, 
provided such a move preserves plane isotopy and both the edges before and after the move are ascending.
To follow a temporal sequence, the edge(s) before a move should not intersect with any outgoing trajectories, except when the edge(s) belong(s) to the original polygonal trajectory.}
\label{fig: triangles}
\end{figure}
\begin{proof}
    Consider a set of projected trajectories among which at least one is non-straight.
    A non-straight projected trajectory bounds a polygon area (illustrated by the area bounded by the solid dark blue and the dashed dark blue lines in Figure \ref{fig: straighten}), which can be partitioned into multiple smaller triangles of $4$ types \cite{prasolov1997knots} (Figure \ref{fig: triangles}): (I) triangles whose interiors contain a crossing between two segments; (II) triangles that contain a vertex of a polygonal trajectory; (III) triangles containing part of a straight segment without any vertex; (IV) those containing an empty space.
    Figure \ref{fig: straighten} illustrates a partitioned polygon.
    Suppose we attempt to shorten a non-straight trajectory by evaluating whether an elementary move can be applied to each of the triangles.
    Two conditions should be satisfied:
    (1) the evaluation of the triangles should follow a temporal sequence, i.e., 
    a triangle containing a later part of a trajectory is evaluated later than a triangle containing an earlier part of the same trajectory;
    (2) the transformed trajectory after each elementary move should be ascending in $t$, i.e., both the edges before and after an elementary move should be ascending.
    Given that the initial trajectories are ascending, there always exists a set of triangles and a sequence of evaluations satisfying both conditions (see Section 2 of the supplementary document for the justification for this statement).
    Figure \ref{fig: straighten} shows a valid sequence of moves with the numbering on each triangle.
    %Such a sequence is always possible because all trajectories are monotonically increasing in $t$.
    If elementary moves can be applied to all triangles, then a non-straight trajectory is isotopic to a straight line.
    Conversely, a trajectory non-isotopic to a straight line must have some triangles to which the elementary moves are not applicable, and we call these triangles tangles.
    By exhaustively listing and assessing all possible forms of the triangles, we find three such tangles (and their symmetric and mirror images), as shown in Figure \ref{fig: tangles}.
    
\begin{figure}
	\centering
    \subcaptionbox{ \footnotesize \label{fig: trianglea}}[0.45\linewidth]{\includegraphics[height=2.7cm]{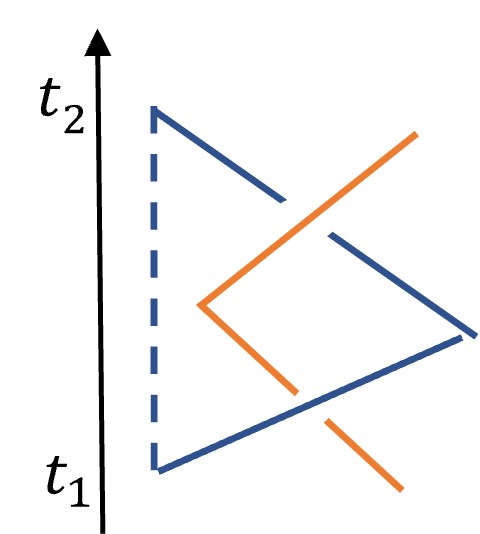}}\\
    \subcaptionbox{\footnotesize \label{fig: triangleb}}[0.45\linewidth]{\includegraphics[height=2.7cm]{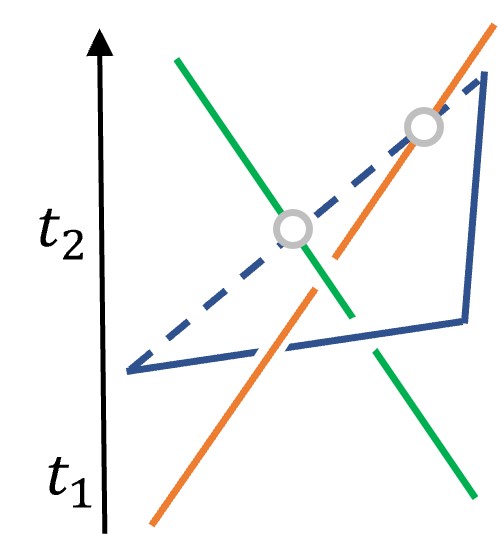}}
    \subcaptionbox{\footnotesize \label{fig: trianglec}}[0.45\linewidth]{\includegraphics[height=2.7cm]{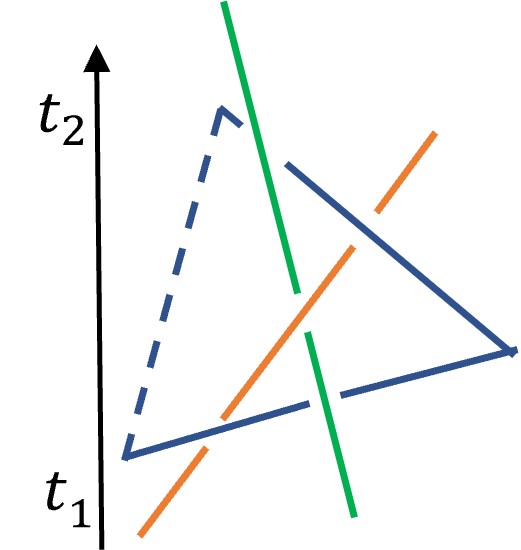}}    
	\caption{\footnotesize Three types of local tangles. The solid blue segments cannot be moved to the dashed segments through plane isotopy. (a) A 2-trajectory tangle with a braid word $\bgen_1\bgen_1$, (b) A 3-trajectory tangle with a braid word  $\bgen_1^{-1}\bgen_2\bgen_1^{-1}$. (c) A 3-trajectory tangle with a braid word $\bgen_1\bgen_2^{-1}\bgen_1\bgen_2^{-1}\bgen_1$.
    }  \label{fig: tangles}
\end{figure}
    Suppose we have applied a sequence of elementary moves on trajectory $i$ and we encounter a tangle the same as Figure \ref{fig: trianglea}, representing an interaction between trajectory $i$ and $j$ from time $t_1$ to $t_2$. 
    Since both trajectories are ascending, we can obtain a braid representation of the trajectories.
    The $2$-braid formed up to time $t_2$, $\word^\projang_{i,j}(t_2)$, is equivalent to $\bgen_1\bgen_1$, because $\word^\projang_{i,j}(t_1)$ has been reduced to identity through previous elementary moves. 
    Similar analysis can be applied to the symmetric and mirror images of Figure \ref{fig: trianglea} to obtain all the representations for $2$-braid tangles, which are $\word^\projang_{i,j}(t)=(\bgen_1\bgen_1)^{\pm 1}$.
    %Clearly, $\bgen_i\bgen_i$ is a braid with a non-trivial closure (a Hopf link).
\begin{figure}[!t]
\centering
\includegraphics[width=0.5\linewidth]{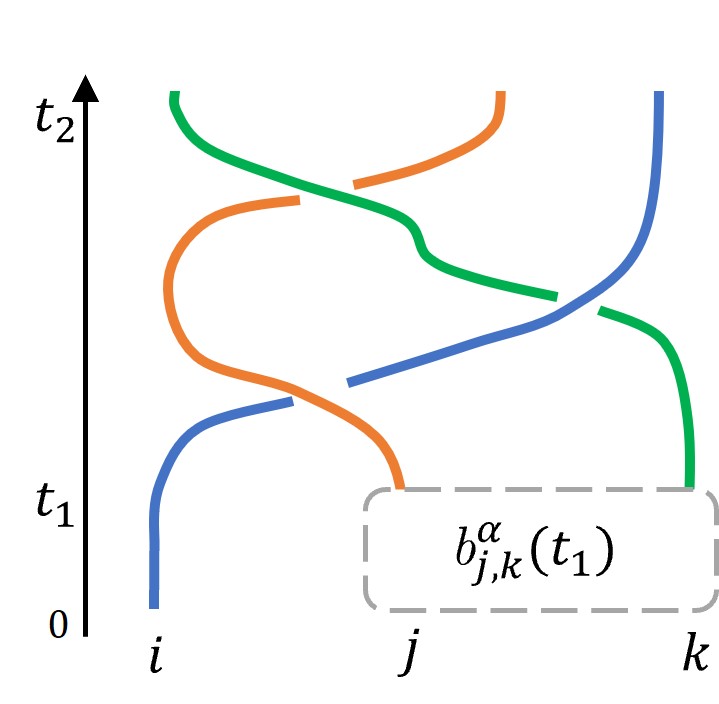}
\caption{\footnotesize The braid diagram containing a $3$-braid tangle in the form of Figure \ref{fig: triangleb}.}
\label{fig: braid3}
\end{figure}

    Suppose we have encountered an interaction among 3 trajectories, $i,j,k\in\myset{I}_n$, the same as Figure \ref{fig: triangleb}. 
    The $3$-braid $\word^\projang_{i,j,k}(t_2)$ is represented as a diagram shown in Figure \ref{fig: braid3}, 
    where $\word^\projang_{j,k}(t_1)$ is a $2$-braid depending on the trajectories of robot $j$ and $k$ up to time $t_1$.
    We first exclude the occurrence of $2$-braid tangles by assuming that the $2$-braid $\word^\projang_{j,k}(t)$ is not equivalent to $(\bgen_1\bgen_1)^{\pm 1}$, $\forall t\in [0,t_2]$.
    This is only possible if $\word^\projang_{j,k}(t_1)$ is equivalent to the identity or $\bgen_1$.
    In the first case, we have $\word^\projang_{i,j,k}(t_2)=\bgen_1^{-1}\bgen_2\bgen_1^{-1}$;
    in the second case, 
    $\word^\projang_{i,j,k}(t_2)=\bgen_2\bgen_1^{-1}\bgen_2\bgen_1^{-1}$, which has a preceding braid $\word^\projang_{i,j,k}(t_0)=\bgen_2\bgen_1^{-1}\bgen_2$ for $t_0<t_2$. 
    %, which is equivalent to the symmetric braid of Figure \ref{fig: triangleb}.
    By applying the same analysis to all symmetric and mirror images of Figure \ref{fig: triangleb} and \ref{fig: trianglec}, and excluding cases of $2$-braid tangles, we obtain the set of words representing the $3$-braid tangles 
    $\{\bgen_\dummyf^{\dummyc}\bgen^{-\dummyc}_\dummyg\bgen_\dummyf^{\dummyc}|\dummyc\in\{1,-1\},\dummyf,\dummyg\in\{1,2\}, \dummyf\neq g \}$.

    %The $2$-braid and $3$-braid tangles may exist among any robot pairs or triplets. 
    Since the braids are invariant to the sequence of robot indices, i.e., $\word_{i,j,k}(t)=\word_{j,i,k}(t)=\word_{k,j,i}(t)$, $\forall i,j,k\in\myset{I}_n$, it is sufficient to consider distinct combinations of robot pairs and triplets
    in the examination of $2$-braids and $3$-braids, hence the condition $i<j<k$.
\end{proof}

Putting all the tools together, we provide sufficient conditions for the avoidance of entanglements.
\begin{thm}\label{thm:nonentangling}
If for all $i,j,k\in\myset{I}_n$, $i<j<k$, $t\in(0,\timet]$, $\projang\in\angleset(\numsample)=\{\frac{i}{\numsample}\pi|i=0\dots \numsample,\numsample>\frac{\pi}{\entangledef},m\in\positiveinteger\}$, 
the $3$-braids and $2$-braids, $\word^\projang_{i,j,k}(t)$ and $\word^\projang_{i,j}(t)$, obtained by projecting the space-time trajectories of $n$ robots onto $\projplane(\projang)$,  satisfies the following:

    (1) $\word^\projang_{i,j}(t)$ is not equivalent to $\bgen_1\bgen_1$ or $\bgen_1^{-1}\bgen_1^{-1}$,
    
    (2) $\word^\projang_{i,j,k}(t)$ is not equivalent to any word in the set 
    $\{\bgen_\dummyf^{\dummyc}\bgen^{-\dummyc}_\dummyg\bgen_\dummyf^{\dummyc}|\dummyc\in\{1,-1\},\dummyf,\dummyg\in\{1,2\}, \dummyf\neq g \}$,\\
    then, the cables of $n$ robots are not $\entangledef$-entangled for all time $t\in[0,\timet]$.
\end{thm}
\begin{proof}
Given that conditions (1) and (2) hold, Lemma \ref{lm:braidword} guarantees that the projected trajectories $\{\sptime_i^\projang\}_{i\in\myset{I}_n}$ are always isotopic to a set of straight lines, $\forall \projang\in\angleset(m)$, $t\in(0,T]$.
Lemma \ref{lm:straightline} ensures that the space-time trajectories are not $\entangledef$-entangled throughout the time interval $(0,T]$.
Finally, due to the isotopy between the cables and the space-time trajectories (Lemma \ref{lm: isotopy}), the theorem is proven.
\end{proof}

\section{Planning Using Permutation Grid}\label{sec:planning}
In this section, we present the approach for path planning of $n$ robots free of $\entangledef$-entanglement for any $\entangledef>\frac{\pi}{2}$.
%We choose $\underline{\entangle}>\frac{\pi}{2}$ because for slack cables, some bendings are allowed...
%Furthermore, this enables a simple and efficient strategy, that is the permutation grid search.
To ensure Lemma \ref{lm:straightline} holds for $\entangledef>\frac{\pi}{2}$, we choose $m=2$ projection axes perpendicular to each other, and obtain the sequence of the robots in increasing order of their projected positions.
%and we name them the first and the second projection axes.
The order of robot $i$ on the $\dummyl$-th projection axis is denoted by $\permpos_i^\dummyl\in\myset{I}_n$, $\dummyl\in\{1,2\}$.
A permutation grid is a $n\times n$ grid space in which each robot takes a position at $(\permpos_i^1,\permpos_i^2)\in\rtwo$,
and none of the robot pairs occupies the same row or column, as shown in Figure \ref{fig: permgrid}.
%The position of each robot on the $\dummyl$-th axis, denoted by $\permpos_i^\dummyl\in\rone$, $\dummyl\in\{1,2\}$, coincides with a grid vertex, and satisfies $\{\permpos_i^\dummyl\}_{i\in\myset{I}_n}=\myset{I}_n$, $\permpos_i^\dummyl\neq\permpos_j^\dummyl$ for $i\neq j$.
In this way, we abstract the Euclidean workspace $\workspace$ into a discrete grid space, and the continuous positions of the robots into permutations.
A move of a robot on the permutation grid always induces an opposite movement of the adjacent robot.
Hence, given a set of robot permutation positions $\permposset=\{\permpos_i^\dummyl|i\in\myset{I}_n,\dummyl\in\{1,2\}\}$, the one-step action space $\permactionset$ consists of exchanging the positions of the adjacent robots, $\permpos_i^\dummyl$ and $\permpos_j^\dummyl$, $\forall \permpos_j^\dummyl=\permpos_i^\dummyl+1, \permpos_i^\dummyl\in\myset{I}_n\backslash n, i,j\in\myset{I}_n, \dummyl\in\{1,2\}$.
Each action represents an elementary $2$-braid $\btau\in\bgen_1^{\pm1}$ added to the $2$-braid $\word_{i,j}^\dummyl$, and an elementary $3$-braid $\btau\in\{\bgen_1^{\pm1},\bgen_2^{\pm1}\}$ added to each $3$-braid involving robot $i$ and $j$, $\word_{i,j,k}^\dummyl$, $k\in\myset{I}_n\backslash\{i,j\}$.

\begin{figure}[!t]
\centering
\includegraphics[width=1.0\linewidth]{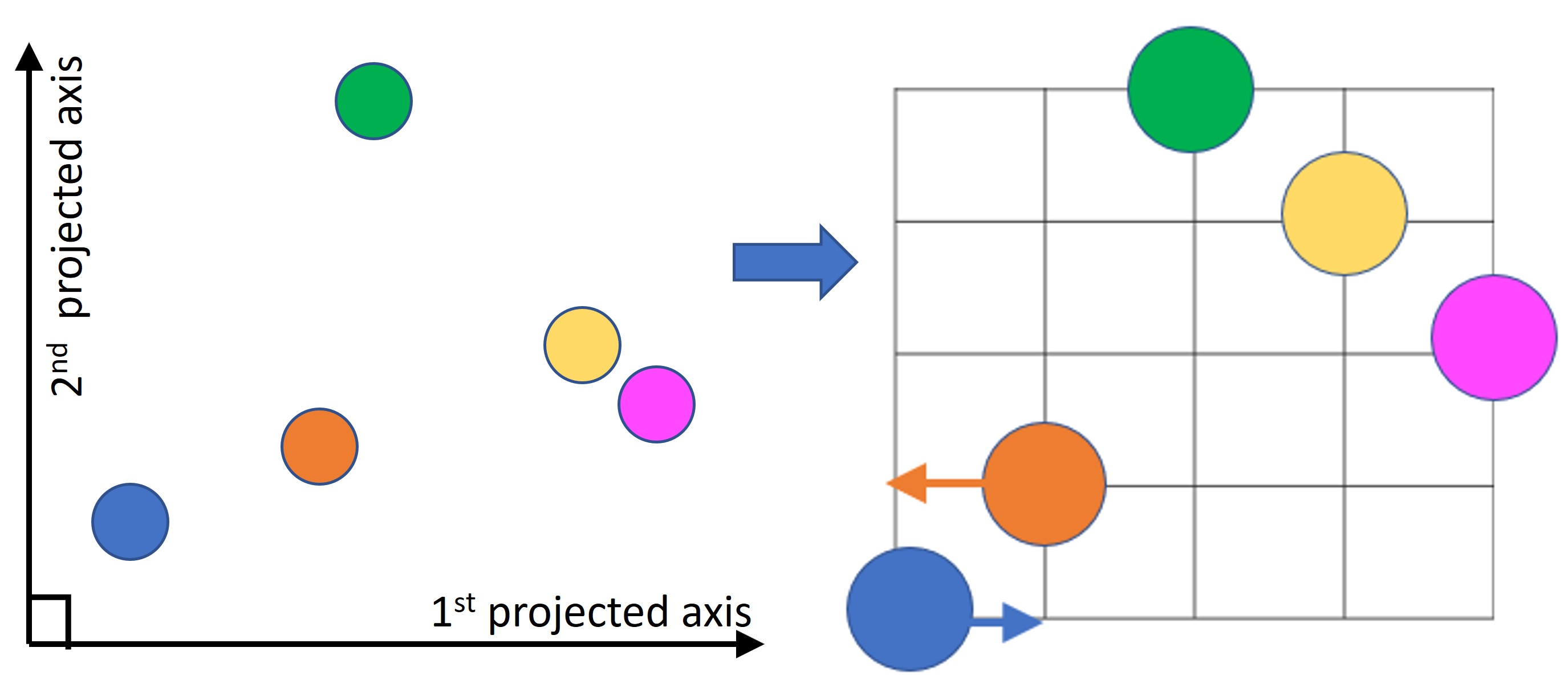}
\caption{\footnotesize Left: the positions of robots in the projected space. Right: a $5\times5$ permutation grid.}
\label{fig: permgrid}
\end{figure}

\begin{algorithm}
\DontPrintSemicolon
\KwIn{initial permutation $\permposset^s$, target permutation $\permposset^d$, initial braids $\bset=\{\word^\dummyl_{i,j},\word^\dummyl_{i,j,k}|{i,j\in\myset{I}_n,i< j<k,\dummyl\in\{1,2\}}\}$}
\KwOut{A path from $\permposset^s$ to $\permposset^d$}
\SetKwBlock{Begin}{function}{end function}
\Begin($\text{GraphSearch} $)
{
  InsertStartingNode(openList,$\permposset^s$,$\bset$)\;
  \While{openList is not empty}
  {
    $\node$ = openList.pop() \tcp{get best node}
    move $\node$ to closedList\;
    \uIf{$\node$.$\permposset=\permposset^d$\label{ln:reachestarget}}
    {
    return RetrievePath()\label{ln:retrievepath}\;
    }
  % \ForAll{$\node.\permpos_j^\dummyl=\node.\permpos_i^\dummyl+1, \node.\permpos_i^\dummyl\in\myset{I}_n\backslash n, i,j\in\myset{I}_n, \dummyl\in\{1,2\}$}
    \ForAll{$\permaction\in\permactionset$\label{ln: permset}}
  {
  $\childnode$ = initializeChild($\node,\permaction$)\label{ln: childinitialize}\;
  % \tcp{initialize child node}
  %    $\childnode=\node$\; 
  %    $\childnode$.parent$=\node$\; 
  % $\childnode.\permpos_i^\dummyl=\node.\permpos_i^\dummyl+1$\;
  % $\childnode.\permpos_j^\dummyl=\node.\permpos_j^\dummyl-1$\;
  $\btau$ = compute2Braid($i,j,\childnode.\permposset$)\label{ln:2braidstart}\;
    [$\childnode.\word_{i,j}^\dummyl$, valid] =
    updateCheck2Braid($\childnode.\word_{i,j}^\dummyl,\btau$)\label{ln:2braidcheck}\label{ln:if2braidvalid}\;
    \uIf{valid == false}
    {
    reject $\childnode$ \label{ln:2braidend}\;
    }  
  \ForAll{$k\in\myset{I}_n,k\neq i,k\neq j$\label{ln:3braidstart}}
  {
    $i,j,k$ = sort($i,j,k$)\;
    $\btau$ = compute3Braid($i,j,k,\childnode.\permposset$)\;
    [$\childnode.\word_{i,j,k}^\dummyl$, valid] = \hspace*{-1mm}updateCheck3Braid($\childnode.\word_{i,j,k}^\dummyl,\btau$)\;
    \uIf{valid == false}
    {
    reject $\childnode$ \label{ln:3braidend}\;
    }
  }
  updateCosts($\childnode$)\;
  \uIf{$\childnode$ in openList}
  {
    Update cost and parent if new cost is lower\;
  }
  \uElseIf{$\childnode$ not in closedList}
  {
  Add $\childnode$ to openList\;
  }
  }    
  }
  return emptyPath\;
  }\label{ln: endentangle}
%   \Return{True}
\caption{Graph search using permutation grid}\label{alg: graphsearch}
\end{algorithm}

%The problem of planning in Euclidean space is converted into planning in a permutation grid from a set of initial permutation positions, $\permposset^s$, to the target permutation positions, $\permposset^d$.
A graph search approach (Algorithm \ref{alg: graphsearch}) is used to generate a feasible path from a set of initial permutation positions, $\permposset^s$, to the target permutation positions, $\permposset^d$.
%using a permutation grid is presented in Algorithm \ref{alg: graphsearch}, which gets as inputs the initial and target grid positions, and the initial 2-braids and $3$-braids representing the existing crossings that have happened in all 2-robot pairs and 3-robot triplets.
Each graph node represents a set of robot positions on the grid, $\permposset$, and carries the $2$-braids and $3$-braids representing the crossing actions that have taken place in all 2-robot pairs and 3-robot triplets.
In every iteration, a node is popped from the open list, and child nodes are generated from the set of permutation actions (line \ref{ln: permset}-\ref{ln: childinitialize}) by exchanging the positions of robot $i$ and robot $j$ on the $\dummyl$-th axis.
Then, the elementary $2$-braid $\btau$ induced by the permutation action is computed, and the word $\word_{i,j}^\dummyl$ is updated and checked against the condition in Theorem \ref{thm:nonentangling} (line \ref{ln:2braidstart}-\ref{ln:2braidcheck}).
A child node that does not satisfy the condition for $2$-braid is rejected immediately.
Similarly, all $3$-braids involving robots $i$ and $j$ are updated and evaluated (line \ref{ln:3braidstart}-\ref{ln:3braidend}).
The heuristic cost is the sum of the Manhattan distances for all robots to reach their targets.
In practice, a bias larger than one is chosen to favor nodes closer to the targets.
The search process continues until a node that reaches $\permposset^d$ is found (line \ref{ln:reachestarget}-\ref{ln:retrievepath}).
The pseudocodes for the updateCheck2Braid and updateCheck3Braid functions are available in Section 3 of the supplementary document.

The output of the search algorithm is a path from the initial permutation to the target permutation, $\permposset^s,\permposset^1,\permposset^2,\dots,\permposset^d$,
which defines a specific topology of the path in the real workspace.
%A function $\mapping:[0,\pi]\times(\myset{I}_n\times\myset{I}_n)\rightarrow\workspace$ can be defined to map the positions on the permutation grid to positions in the real workspace $\workspace$.
In our approach, we use a simple linear function $\mapping:[0,\pi]\times(\myset{I}_n\times\myset{I}_n)\rightarrow\workspace$ to map the permutation grid to a $n\times n$ grid in the workspace, where the grid size is larger than a safety distance between robots to ensure collision avoidance.
Hence, the robots follow a set of waypoints $\mapping(\alpha,\permposset^s),\mapping(\alpha,\permposset^1),\dots\mapping(\alpha,\permposset^d)$, and finally move to $\{\config_i^d\}_{i\in\myset{I}_n}$ using straight paths.
During the movements of robots, the $2$-braids and $3$-braids are updated when crossings between robots take place.
These updated braid words can be used as initial conditions for subsequent planning with guaranteed entanglement avoidance.
%In practice, the projection angles are chosen as $\projang\in\{0,\frac{1}{2}\pi\}$,

\section{Simulations}\label{sec:simulation}
Simulations of multiple tethered UAVs are conducted in Unity game editor with AGX Dynamics plugin \footnote{https://www.algoryx.se/agx-dynamics/} installed to accurately simulate the dynamics of the cables and the effect of entanglements on the robots.
A slack and non-retractable cable of fixed length is attached to each simulated UAV.
The proposed permutation grid planning algorithm is implemented as a Robot Operating System (ROS) program which transmits the planned waypoints to the simulator through ROS TCP Connector \footnote{https://github.com/Unity-Technologies/ROS-TCP-Connector}.
In each simulation run, a team of $6$ to $10$ UAVs is tasked to travel to $100$ sets of target positions.
A task is successful when all robots reach their assigned targets;
if a set of targets cannot be reached or some of the robots are stuck, the task fails and the targets will be updated.

The following existing approaches are used for comparison: (1) Neptune \cite{cao2022neptune}, a distributed trajectory planning approach for multiple tethered robots considering a slack cable model; (2) Hert \cite{hert1999motion}, a centralized approach considering fully stretched cables; 
(3) a baseline multi-robot trajectory planner without the cable-related constraints, labeled as Neptune* 
(the codes for both Neptune and Neptune* are released by the authors of \cite{cao2022neptune}).
In all simulations, the proposed and the compared algorithms run on a mini-computer with Intel i7-8550U CPU.
Two screenshots of simulations are shown in Figure \ref{fig: simulation}.
A video of the simulations can be viewed in the supplementary material.

\begin{figure}[!t]
\centering
\includegraphics[width=1.0\linewidth]{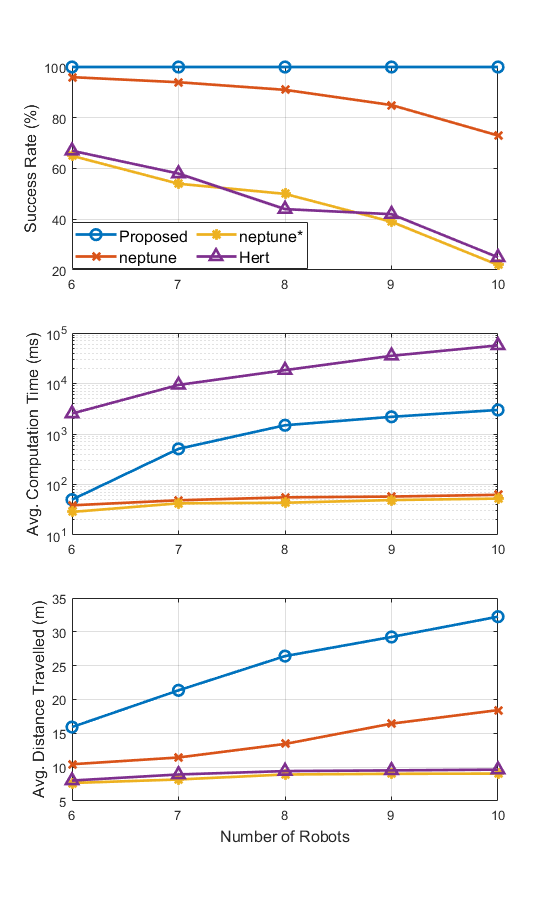}
\caption{\footnotesize  Plots of the success rate, the average computation time, and the average distance traveled for the proposed and the compared approaches.}
\label{fig: benchmark}
\end{figure}

Figure \ref{fig: benchmark} shows the success rate, the computation time, and the distance traveled for all approaches with respect to the number of robots involved.
From the top plot, we can observe that the proposed approach is the only one to ensure all tasks are successfully completed.
The success rate of Neptune is close to $100\%$ with $6$ robots but drops to below $80\%$ when the number of robots increases to $10$.
As a distributed approach, Neptune is unable to guarantee feasible and entanglement-free paths for all robots, hence freezing robots are observed in the simulation.
The success rates of Neptune* and Hert are significantly lower than the other two approaches.
In both approaches, the entanglement of the cables accumulates and results in a huge tangle in the center of the workspace (Figure \ref{fig: gotentangle}).
Eventually, only the targets near the tangle can be reached by the robots.
Hert generates 3-D paths that require a robot to move below a taut cable to avoid cable contacts.
%with a lot of vertical movements to avoid cable contacts between taut cables.
However, in the case of a slack cable model, 
moving vertically does not generate the same path topology as in the case of taut cables, because slack cables lie on the ground.
%The ineffectiveness of Hert in addressing the slack cable model is due to its reliance on vertical movements to avoid cable contacts.
%When the cables lie on the ground, 
Hence, the performance of Hert is only comparable to a baseline multi-robot planner where cables are not considered at all.

The middle plot shows the computation time of all approaches. 
Both distributed approaches generate initial trajectories within $100$ms, but the generated trajectories only ensure collision avoidance for a short planning horizon, and frequent online replanning is required.
Our approach generates trajectories for all robots in a one-time computation.
The computation time increases with the number of robots, but an average computation time of $3$s for $10$ robots is acceptable as a waiting time for on-demand targets.
Although both our approach and Hert are centralized approaches, we rely on pre-computed reduction rules to efficiently update braids and check entanglement.
On the other hand, the computation time of Hert is burdened by the expensive line-triangle intersection checking procedure to avoid cable-cable contacts.

One weakness of the proposed approach is the length of the generated paths, as seen from the bottom plot.
The average distance traveled for each robot (only considering successful tasks) is considerably higher in our approach.
This is due to the direct mapping of paths in a permutation grid into the real workspace.
Every movement of a robot accompanies an opposite movement of another robot, which could be unnecessary and increases the distance traveled by each robot.
Although direct mapping is an inefficient strategy, it is a simple implementation to validate the effectiveness of the proposed approach in generating entanglement-free paths.
In our future work, efficient topology-guided path generation will be studied and integrated with the proposed permutation grid search.

\section{Flight Experiment}\label{sec: exp}
We verify the practicality of the proposed approach using three small tethered UAVs in a $5$m$\times5$m$\times2$m indoor area.
Each UAV is connected to a ground power supply using a long power cable.
Random targets are generated during the experiment, and a ground computer computes the paths and sends them to the UAVs through a Wifi network.
Figure \ref{fig: experiment} illustrates a flight experiment, where the robots and their cables are highlighted for easy identification.
The supplementary video shows an experiment in which three UAVs complete $25$ sets of targets successfully and remain untangled.
The average computation time on a computer with Intel i7-8750H CPU is less than $10$ms, which guarantees the online performance of the algorithm.
The average completion time for a set of targets is $12.5$s with $0.7$m/s maximum velocity of the UAVs.

\begin{figure}[!t]
\centering
\includegraphics[width=1.0\linewidth]{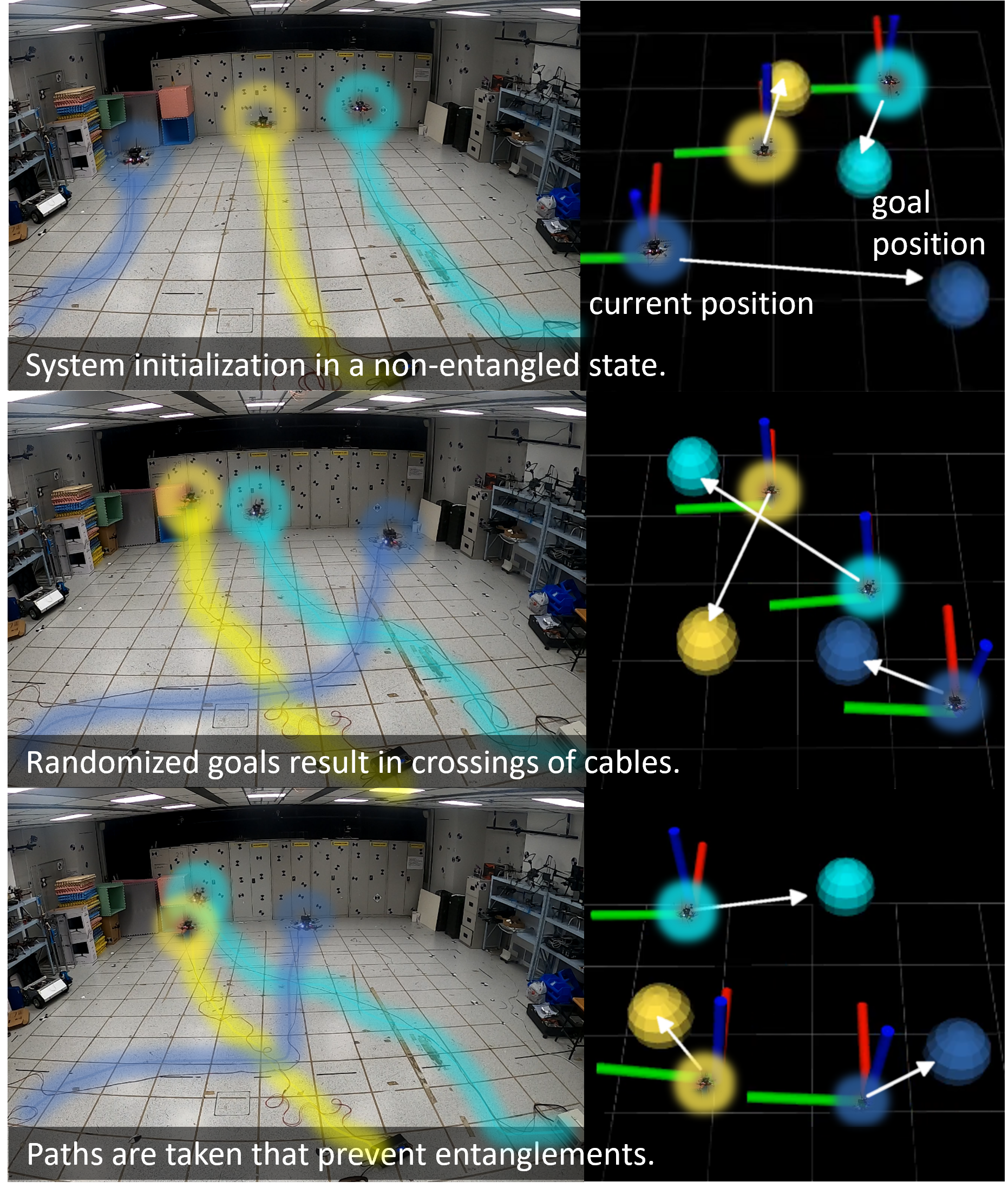}
\caption{\footnotesize Left: photos of a flight experiment. Right: visualization of positions and goal points of the robots.}
\label{fig: experiment}
\end{figure}

\section{Conclusion} 
\label{sec:conclusion}
In this work, we have investigated the problem of path planning for multiple tethered robots.
The main contribution of this work is to establish the connection between the entanglements of the cables and the topological braids representing robots' trajectories.
This is accomplished by (1) showing the topological equivalence between the cables and the robots' space-time trajectories, (2) converting the projected space-time trajectories into braids, and (3) identifying particular braid patterns that are necessary for the occurrence of entanglements.
%To generate non-entangling paths, feasible topologies are searched in a graph of permutation states, 
A graph search algorithm based on the permutation grid has been proposed for generating a feasible topology of robot paths,
and paths containing the entangling braids patterns are guaranteed to be rejected.
Simulations and experiments demonstrate the effectiveness of the proposed algorithm in avoiding entanglement in complex and realistic scenarios.
To address the issue of long path length highlighted by the simulation results,  
our future research will focus on efficient multi-robot path generation in Euclidean space given a specific path topology.
 \section*{Acknowledgments}
 This research was conducted under project WP5 within the Delta-NTU Corporate Lab with funding support from A*STAR under its IAF-ICP programme (Grant no: I2201E0013) and Delta Electronics Inc, and the Wallenbery-NTU Presidential Postdoctoral Fellowship in Nanyang Technological University, Singapore.
 We thank Dr. Fedor Duzhin for providing valuable suggestions to improve the manuscript and Mr. Xinhang Xu for providing the hardware platform for simulation.

\bibliographystyle{plainnat}
\bibliography{references}

\end{document}